%% file: main.tex
\documentclass[twoside]{article}

\usepackage[accepted]{aistats2024}
\usepackage{mathtools}
\usepackage{amssymb, amsmath, amsthm}
\usepackage{dsfont}
\usepackage{hyperref}
\hypersetup{colorlinks,citecolor=blue}
\usepackage{cleveref}

\usepackage[round]{natbib}

\usepackage{subcaption}
\usepackage{bm}
\usepackage{multirow}
\usepackage{algorithm}
\usepackage{algorithmic}
\input{macro}

\begin{document}

%

%

\twocolumn[

\aistatstitle{Near-Optimal Policy Optimization for Correlated Equilibrium in General-Sum Markov Games}

\aistatsauthor{ Yang Cai \And Haipeng Luo \And  Chen-Yu Wei \And Weiqiang Zheng}

\aistatsaddress{ Yale University \And  University of Southern California \And University of Virginia \And Yale University} ]

\begin{abstract}
    We study policy optimization algorithms for computing correlated equilibria in multi-player general-sum Markov Games. Previous results achieve $\Tilde{O}(T^{-1/2})$ convergence rate to a correlated equilibrium and an accelerated $\Tilde{O}(T^{-3/4})$ convergence rate to the weaker notion of coarse correlated equilibrium. In this paper, we improve both results significantly by providing an uncoupled policy optimization algorithm that attains a near-optimal $\Tilde{O}(T^{-1})$ convergence rate for computing a correlated equilibrium. Our algorithm is constructed by combining two main elements (i) smooth value updates and (ii) the \emph{optimistic-follow-the-regularized-leader} algorithm with the log barrier regularizer. 
\end{abstract}

\section{Introduction}
How does a multi-agent system evolve when each agent independently updates their policy based on their own utility? 
Can the system converge to an equilibrium, and if so, how quickly?
These questions lie at the heart of game theory, economics, and learning theory, and have stimulated decades of research. For example, in normal-form games, it is well-known that when each agent employs a standard online learning algorithm with low external regret or low swap regret, the empirical distribution of their joint strategy profile converges to a coarse correlated equilibrium (CCE) or a correlated equilibrium (CE) respectively.

While $\sqrt{T}$ (external/swap) regret is minimax optimal after $T$ interactions in the adversarial environment, it is possible to achieve strictly better regret in a normal-form game when each agent employs the same no-regret algorithm.
For example, \cite{syrgkanis2015fast} show that $T^{1/4}$ external regret can be achieved by the \emph{optimistic} online mirror descent (OOMD) algorithm or the \emph{optimistic} follow-the-regularized-leader (OFTRL) algorithm.
Various improvements on this result have been proposed over the past few years, with the most recent one by \citet{anagnostides2022near-optimal, anagnostides2022uncoupled} achieving a near-optimal $\log(T)$ bound for both the external and swap regret. A direct corollary of this result is that when all agents employ the corresponding algorithm, the empirical distribution of their joint strategy profile converges to a CCE or CE, respectively, at a rate of $\Tilde{O}(T^{-1})$.

However, achieving similar results for the more general setting of Markov games -- the focus of this work, is much more challenging. Importantly, unlike the normal-form game setting, achieving $o(T)$ regret has been shown to be both statistically and computationally intractable for Markov games~\citep{tian2021online,foster23hardness}. 
This significant difference leads to considerably different algorithms for Markov games, the majority of which are aimed at finding an approximate equilibrium directly. 
We review this line of work in \Cref{sec:related_work} and only point out here that, with an oracle access to the reward and transition function of the Markov game (see Remark \ref{remark:implementation}), the state-of-the-art uncoupled learning dynamic converges to a CCE at a rate of $T^{-3/4}$~\citep{zhang2022policy} and to a CE at a rate of $T^{-1/2}$~\citep{jin2021v, song2021can, mao2023provably}, both of which are substantially slower than the aforementioned $\Tilde{O}(T^{-1})$ rate for normal-form games.

In this work, we close this gap by proposing an uncoupled policy optimization algorithm that converges to a CE (thus also to the weaker notion of CCE) at a near-optimal rate of $\log^2(T)/T = \Tilde{O}(T^{-1})$, significantly improving existing results.
Our algorithm builds upon the OFTRL framework with smooth value updates similar to \citet{zhang2022policy}, but importantly also incorporates the technique of using the log barrier as a regularizer from the recent work of~\citet{anagnostides2022uncoupled}.

\subsection{Related Work}\label{sec:related_work}

\paragraph{Learning in normal-form games}
The connection between no-regret learning algorithms and finding equilibria in games dates back to the seminal work of~\citet{freund1999adaptive}, which shows that in a two-player zero-sum games, if both players have an external regret bound of $R$ after $T$ rounds, then their average strategy is an $R/T$-approximate Nash equilibrium (NE).
For general-sum games, similar connections hold between the external regret and CCE, and also between the stronger notion of swap regret and CE~\citep{stoltz2007learning, blum2007external}.

As mentioned, while in the worst case the best possible regret bounds are of order $\sqrt{T}$, the players could enjoy even lower regret when all of them employ the same algorithm, since this usually leads to an overall stable environment.
Such results were pioneered by \cite{daskalakis2011near, rakhlin2013optimization} for two-player zero-sum games, and extended by \cite{syrgkanis2015fast} for general-sum games. 
Various improvements have been made over the past few years~\citep{chen2020hedging, daskalakis2021near-optimal, anagnostides2022near-optimal, anagnostides2022uncoupled}.

\paragraph{Learning in Markov games}
Markov game \citep{shapley1953stochastic} is a general framework for modeling multi-agent sequential decision making problems. A line of earlier development has already focused on designing decentralized learning algorithms that offer better scalability \citep{littman1994markov, littman2001friend, bowling2001rational}, but the convergence guarantees are often only asymptotic.  
Inspired by recent advances in online optimization and the empirical success of multi-agent reinforcement learning through self-play \citep{silver2017mastering, vinyals2019alphastar, bard2020hanabi}, there is a surge of research trying to sharpen the theoretical guarantees for multi-agent learning in Markov games, especially in the decentralized setting \citep{bai2020near, wei2021last, sayin2021decentralized, mao2023provably, jin2021v, song2021can, kao2022decentralized, leonardos2021global, ding2022independent, erez2023regret, cui2023breaking, wang2023breaking}.  
Below, we only highlight the most relevant ones.

As mentioned, \citet{tian2021online} show that no-regret learning is generally impossible when against arbitrary opponents. 
However, this does not preclude the possibility of enjoying low regret against Markov policies when all players employ the same algorithm.
Indeed, \citet{erez2023regret} design a policy optimization algorithm (also using techniques from \citet{anagnostides2022uncoupled}) that achieves $\Tilde{O}(T^{3/4})$ swap regret (a weaker notion of swap regret that concerns only Markov policy deviations) assuming the same oracle access to reward/transition as we do.
This implies $\Tilde{O}(T^{-1/4})$ convergence to a certain kind of CE that only allows Markov policy deviation, a notion weaker than ours.

Recent findings by \cite{foster23hardness} indicate that the previously mentioned results by \cite{erez2023regret} is unlikely to hold when general deviations are permitted. More explicitly, under standard computational complexity assumptions,\footnote{This is based on the assumption that PPAD-hard problems are not solvable in polynomial time.} no polynomial-time algorithm can be no-regret in general-sum Markov games when executed independently by all players, even if the algorithm designer knows the game.

Due to such impossibility results, most algorithms directly aim at finding CCE/CE without considering the regret of the players. 
Specifically, the certified policy output by the V-learning algorithm \citep{jin2021v, song2021can, mao2023provably} has been proven to converge to a CCE/CE at a rate of $T^{-1/2}$, which is optimal when the players need to learn all the game parameters through interactions with the environment. In the simpler scenario where an oracle access to reward/transition is given, the best currently known rate is $T^{-3/4}$ for finding a CCE~\citep{zhang2022policy}, and $T^{-1/2}$ for finding a CE (still by V-learning). In this work, we improve both rates to $\log^2(T)/T$. 

\section{Preliminaries}
For a positive integer $n$, we denote the set $\{1,2,\ldots,n \}$ as $[n]$. For any set $\+A$, the probability simplex over $\+A$ is $\Delta_{\+A}:= \{ x \in \R^{|\+A|}: \sum_{a \in \+A} x[a] = 1, x[a] \ge 0, \forall a \in \+A \}$.
\paragraph{Multi-player General-Sum Markov Games} In this paper, we focus on finite-horizon $n$-player general-sum Markov games denoted as $\+M(H, \+S, \{\+A_i\}_{i\in [n]}, \-P, \{r_i\}_{i \in [n]})$, where $H$ is the length of the horizon; $\+S$ is the set of states with size $|\+S| = S$; $A_i$ is the the action set of player $i$ with size $|\+A_i| = A_i$ and we denote a joint action profile of all players as $\bolda = (a_1, a_2, \ldots, a_n) \in \Pi_{i=1}^n \+A_i$; $\-P = \{\-P_h\}_{h \in [H]}$ is the transition probabilities where $\-P_h(s' \mid s, \bolda)$ specifies the probabilities of transition to state $s'$ in step $h+1$ if players take the joint action $\bolda$ at $s$ in step $h$; $r_i = \{r_{i,h}\}_{h \in [H]}$ are the reward function for player $i$ where $r_{i,h}(s, \bolda) \in [0,1]$ is the reward for player $i$ when players take the joint action $\bolda$ at $s$ in step $h$. In each episode, we assume the game starts at $s_1$ without loss of generality. In each step $h \in [H]$, each player observes the current state $s_h$ and chooses an action $a_{i,h} \in \+A_i$, then each player receives reward $r_{i,h}(s_h, \bolda_h)$ and the game transits to the next state $s_{h+1} \sim \-P_h(\cdot \mid s_h, \bolda_h)$. The episode ends after $H$ steps. 

\paragraph{Policies and Value Functions} 
A (random) policy $\pi_i$ for player $i$ is a collection of $H$ maps $\{\pi_{i,h}: \Omega \times (\+S \times \Pi_{i=1}^n \+A_i)^{h-1} \times \+S \rightarrow \Delta_{\+A_i}\}_{h \in [H]}$ where $\pi_{i,h}$ maps a random sample $\omega$ from a probability space, a history of length $h-1$, and the current state to a probability distribution (mixed strategy) over $\+A_i$. To execute policy $\pi_i$, player $i$ samples $\omega$ at the beginning of the episode, then at each step $h$, supposing the history is $\tau_h := (s_1, \bolda, \ldots, s_{j-1}, \bolda_{h-1})$, player $i$ chooses action $a_{i,h} \sim \pi_{i,h}(\cdot \mid \omega, \tau_h, s_h)$. We note that $\omega$ is shared across all steps $h \in [H]$. A \emph{Markov policy} for player $i$ is collection of $H$ history independent maps $\pi_i = \{ \pi_{i,h}: \+S \rightarrow \Delta_{\+A_i}\}$, where $\pi_{i,h}(a \mid s_h)$ specifies the probability of taking action $a \in \+A_i$ at $(h, s_h)$.  

A joint policy $\pi$ is a set of policies denoted as $\pi = \pi_1 \odot \pi_2 \odot \ldots \odot \pi_n$ where the same random sample $\omega$ is shared among all players. We denote $\pi_{-i}:= \pi_1 \odot \ldots \pi_{i-1} \odot \pi_{i+1} \odot \ldots \odot\pi_{n}$ the joint policy that excludes player $i$. When the random sample has a special form $\omega = (\omega_1, \ldots, \omega_n)$ and for each $i \in [n]$, $\pi_i$ only uses the randomness in $\omega_i$ that is independent of $\{\omega_j\}_{j\ne i}$, then the joint policy is a \emph{product policy} and we denote it as $\pi = \pi_1 \times \pi_2 \times \ldots \times \pi_n$. We denote $\pi_{h}(\bolda|s)$ the probability of joint action $\bolda$ at state $s$. The value function $V_{i,h}^\pi: \+S \rightarrow \R$ specifies the expected reward for player $i$ if all players follow policy $\pi$:
\[
V^\pi_{i,h}(s) := \-E_{\pi}\InBrackets{\sum_{h' = h}^H r_{i,h'}(s_{h'}, \bolda_{h'}) \mid s_h = s}.
\]
The goal of player $i$ is to maximize their own value function $V_{1,h}^\pi(s_1)$. The $Q$ function at step $h$ is defined as 
\[
Q^\pi_{i,h}(s, \bolda) := \-E_{\pi}\InBrackets{\sum_{h' = h}^H r_{i,h'}(s_{h'}, \bolda_{h'}) \mid s_h = s, \bolda_h = \bolda}.
\]

\paragraph{Strategy Modification and Correlated Equilibrium}
A strategy modification $\phi_i$ for player $i$ is a collection of maps $\phi_i=\{\phi_{i,h}: (\+S \times \Pi_{i=1}^n \+A_i)^{h-1} \times \+S \times \Delta_{\+A_i}  \rightarrow \Delta_{\+A_i}\}$ such that given history $\tau_h$ and state $s_h$, each map $\phi_{i,h}(\tau_h, s_h, \cdot): \Delta_{\+A_i} \rightarrow \Delta_{\+A_i}$ is a linear transformation.\footnote{On the other hand, the set of strategy modifications studied in \citep{erez2023regret} is $\{\phi_{i}: \+S \times \+A_i \rightarrow \+A_i\}$, which is a strict subset of ours and thus induces a weaker notion of correlated equilibrium.} For any policy $\pi_i$, the modified policy denoted as $\phi_i \diamond \pi_i$ changes the strategy $\pi_{i,h}(\omega, \tau_h, s_h)$ under random sample $\omega$ and history $\tau_h$ to another strategy $\phi_{i,h}(\tau_h, s_h, \pi_{i,h}(\omega, \tau_h, s_h))$.  

A correlated equilibrium is a joint policy where no player can increase their value by any strategy modification. Formally, it is defined as
\begin{definition}[Correlated Equilibrium]
    A joint policy $\pi$ is a correlated equilibrium (CE) if $\max_{i\in [n]} \max_{\phi_i} V_{i,1}^{(\phi_i \diamond \pi_i)\odot\pi_{-i}}(s_1) - V_{i,1}^{\pi}(s_1) \le 0$.  A joint policy $\pi$ is an $\epsilon$-approximate CE if $\cegap(\pi):=\max_{i\in [n]} \max_{\phi_i} V_{i,1}^{(\phi_i \diamond \pi_i)\odot \pi_{-i}}(s_1) - V_{i,1}^{\pi}(s_1) \le \epsilon$. 
\end{definition} 
A coarse correlated equilibrium is a joint policy where no player can increase their value by playing any other independent policy. Formally, it is defined as
\begin{definition}[Coarse Correlated Equilibrium]
   A joint policy $\pi$ is an $\epsilon$-approximate coarse correlated equilibrium if $\max_{i\in [n]} \max_{\pi_i'} V_{i,1}^{\pi_i'\times\pi_{-i}}(s_1) - V_{i,1}^{\pi}(s_1) \le~\epsilon$. 
\end{definition} 
We remark that by definition a CE is also a CCE. In the rest of the paper, we focus on CE only, but the same results apply to CCE clearly.

\paragraph{Additional Notations} 
Define $A_{\max} = \max_{i\in[n]} A_i$.
For any value function $V: \+S \rightarrow \R$, we define $[\-P_h V](s,\bolda):= \-E_{s' \sim \-P_h(s,\bolda)} V(s')$. For any Markov policy $\pi_h(\cdot \mid s)$ and any $Q$ function $Q_{i,h}(\cdot, \cdot): \+S \times \Pi_{j=1}^n \+A_j \rightarrow \R$, we denote $[Q_{i,h}\pi_{h}](s):= \InAngles{Q_{i,h}(s, \cdot), \pi_{h}(\cdot \mid s)}$. Similarly, for any joint policy $\pi_{-i, h}(\cdot \mid s)$ that excludes player $i$, we denote $[Q_{i,h}\pi_{-i,h}](s, a_i):= \InAngles{Q_{i,h}(s, a_i, \cdot), \pi_{-i, h}(\cdot \mid s)}$.

\subsection{Online Learning and Regret}\label{sec:OL}
In a (linear) online learning setting, at each day $t \in \-N$, the learner chooses a strategy $x^t$ from a compact and convex set $\X \subseteq \R^d$ while the adversary picks a reward vector $u^t \in \R^d$.  Then the learner gets reward $\InAngles{u^t, x^t}$ and the reward vector $u^t$ as feedback. The goal of an online learning algorithm is to minimize \emph{regret}, or more generally, \emph{$\Phi$-regret}. For a set of strategy modifications $\Phi = \{\phi: \X \rightarrow \X\}$, the $\Phi$-regret of an algorithm $\mathfrak{R}$ over a time horizon $T$ is defined as 
\[
\reg^T_\Phi:= \max_{\phi \in \Phi}\sum_{t=1}^T \InAngles{u^t, \phi(x^t) - x^t}.
\]
An algorithm is \emph{no $\Phi$-regret} if its $\Phi$-regret is sublinear in $T$. The \emph{(external) regret} denoted as $\Reg^T$ is $\Phi$-regret when $\Phi$ includes only constant transformations. The \emph{swap regret} denoted as $\sreg^T$ is $\Phi$-regret when $\Phi$ includes all possible linear transformations. The swap regret is non-negative since we can choose the identity transformation such that $\phi(x) = x$ for all $x\in \X$.

\section{Algorithm and Main Results}

\begin{algorithm}
\caption{Policy optimization in Markov games with $V$ value update}
    \begin{algorithmic}[1]\label{alg:main alg V update}
        \REQUIRE step size $\eta > 0$, weights $\{\alpha_t\}$ and $\{w_t\}$, an online learning algorithm $\mathfrak{R}$.
        \STATE \textbf{Initialize:} For all $(i,h,s)$, initialize $V^{0}_{i,h}(s)= H-h+1$,  $\mathfrak{R}_{i,h,s}$ as an instance of $\mathfrak{R}$ over decision set $\Delta_{\+A_i}$, and $\pi^0_{i,h}(\cdot \mid s)$ as $\mathfrak{R}_{i,h,s}$'s initial output
        \FOR{$t = 1, 2, \ldots, T$}
        \FOR{all $(i,s,h)$} 
        \STATE Forward the utility vector $u^{t-1}_{i,h,s}:=  \frac{w_{t-1}}{H} [(r_h +  \mathbb{P}_h V^{t-1}_{i,h+1}) \pi^{t-1}_{-i,h}](s, \cdot) $ to $\mathfrak{R}_{i,h,s}$. 
        \STATE Update $\pi^t_{i,h}(\cdot \mid s)$ according to $\mathfrak{R}_{i,h,s}$.
        \ENDFOR
        \STATE for all $(i,s,\bolda) \in [n] \times \+S \times \+A$, from $h = H$ to $1$: \label{line: smooth value update}
        \begin{align*}
            &V^t_{i,h}(s) \\
            &\leftarrow (1 - \alpha_t) V^{t-1}_{i,h}(s)  + \alpha_t \InBrackets{(r_h + \mathbb{P}_h V^t_{i,h+1})\pi^t_{h} }(s).
        \end{align*}
        \ENDFOR
    \end{algorithmic}
    Output $\hpi^T = \hpi^T_1$ as defined in \Cref{alg:output policy}. 
\end{algorithm}

\begin{algorithm}[ht]
    \caption{Executing Policy $\hpi^t_h$}
    \begin{algorithmic}[1]\label{alg:output policy}
        \REQUIRE Product policies $\pi^{t'}_{h'}(\cdot \mid s') = \Pi_{i=1}^n \pi^{t'}_{i,h'}(\cdot \mid s')$ for all $(h',s',t') \in [H] \times \+S \times [T]$.
        \STATE Sample $j \in [t]$ with probability $\Pr[j = i] = \alpha_j^i$ (see \Cref{eq:alpha_t^i} for definition).
        \STATE Play policy $\pi^j_h$ at the $h$-th step of the game.
        \STATE Play policy $\hpi^{j}_{h+1}$ for step $h+1$.
    \end{algorithmic}
\end{algorithm}

In this section, we present a policy optimization algorithm (\Cref{alg:main alg V update}) for learning correlated equilibrium in multi-player general-sum Markov games.  \Cref{alg:main alg V update} is a single-loop algorithm where on each step-state pair $(h,s) \in [H] \times [\+S]$, each player employs a no-regret algorithm over its own action set following the online learning protocol described in \Cref{sec:OL} with some reward vectors carefully constructed from a smooth update. We explain both the value update and the policy update below.

\subsection{Value Update}\label{sec:value update}
Each player maintains $V$ value function $V^t_{i,h}$ and conducts \emph{smooth value update} with the following learning rates (Line~\ref{line: smooth value update} of Algorithm~\ref{alg:main alg V update}):
\begin{equation}\label{eq:alpha_t}
    \alpha_t = \frac{H+1}{H+t}.
\end{equation}
The choice of $\alpha_t = O(\frac{1}{t})$ is proposed by \citet{jin2018q} and adopted in many subsequent works~\citep{jin2021v,wei2021last,zhang2022policy,yang2023ot}. This choice ensures conservative updates of value functions and hence stabilizes the update of policies. 
We also define a group of auxiliary weights:
\begin{equation}\label{eq:alpha_t^i}
    \alpha_t^t = \alpha_t, \quad \alpha_t^i = \alpha_i \Pi_{j=i+1}^t (1-\alpha_j),  \forall i \le t-1,
\end{equation}
and 
\begin{equation}\label{eq:w_t}
    w_0 = w_1, \quad  w_t = \frac{\alpha_t^t}{\alpha_t^1}, \quad \forall t \ge 1.
\end{equation}
After $T \ge 1$ episodes, \Cref{alg:main alg V update} outputs a joint policy $\hpi^T$ as defined in \Cref{alg:output policy}. The output policy is not a Markov policy and is defined recursively. Specifically, at each step $h$, the policy $\hpi^t_h$ randomly selects a product policy from $\{\pi^j_h\}_{j\in [t]}$ with probability $\{\alpha_t^j\}_{j\in [t]}$ and plays policy $\hpi^j_{h+1}$ onward. 

\begin{remark}\label{remark:implementation}
    \Cref{alg:main alg V update} is an adaptation of \citep[Algorithm 12]{zhang2022policy}, a policy optimization algorithm originally designed for learning the coarse correlated equilibrium. The original algorithm performs $Q$ value update, whereas our adaptation focuses on $V$ value update. An equivalent version of \Cref{alg:main alg V update} that employs $Q$ value update is presented in \Cref{alg:main alg}, with its equivalence proven in \Cref{prop:V=Q}.
    
    The main distinction between the two algorithms lies in their function sizes. The $Q(s, \bolda)$ function has a size of $S \cdot \Pi_{i=1}^n A_i$, which grows exponentially with the number of agents, leading to the so-called curse of multi-agents. In contrast, the $V(s)$ function is significantly more compact with a size of $S$, effectively bypassing the curse of multi-agents~\citep{jin2021v}.
    
    Furthermore, \Cref{alg:main alg V update} offers a notable advantage: it supports a \emph{decentralized} implementation. This means each player does not need explicit knowledge of other players' policies. The update steps in \Cref{alg:main alg V update} only require $(r_h +\-P_h V_i)\pi_{-i,h}$ for any value function $V_i$ 
    and the policies of other players $\pi_{-i,h}$. 
    This  can be efficiently computed with access to:
    \begin{itemize}
        \item[1.] A reward oracle that provides the expected reward vector for player $i$ based on the policies of other players $\pi_{-i,h}(\cdot \mid s)$ at $(h,s)$.
        \item[2.] A transition oracle that offers the distribution of $s_{h+1}$ based on player $i$'s action $a_i$ and the policies of other players $\pi_{-i,h}(\cdot \mid s)$ at $(h,s)$.
    \end{itemize}
    While we directly assume access to such oracles, both of them can be approximately implemented within $\varepsilon > 0$ error using $\poly(n, A_{\max}, S, H, 1/\varepsilon)$ samples.
    
\end{remark} 
Given the equivalence between \Cref{alg:main alg V update} and \Cref{alg:main alg}, any guarantee for \Cref{alg:main alg} also holds for \Cref{alg:main alg V update}.  We will thus focus on \Cref{alg:main alg} in the rest of the paper.

\begin{algorithm}[ht]
\caption{Policy optimization in Markov games with $Q$ value update~\citep{zhang2022policy}}
    \begin{algorithmic}[1]\label{alg:main alg}
        \REQUIRE step size $\eta > 0$, weights $\{\alpha_t\}$ and $\{w_t\}$, an online learning algorithm $\mathfrak{R}$
        \STATE \textbf{Initialize:} For all $(i,h,s)$, initialize $Q^{0}_{i,h}(s, \bolda)= H-h+1$, $\mathfrak{R}_{i,h,s}$ as an instance of $\mathfrak{R}$ over decision set $\Delta_{\+A_i}$, and $\pi^0_{i,h}(\cdot \mid s)$ as $\mathfrak{R}_{i,h,s}$'s initial output
        \FOR{$t = 1, 2, \ldots, T$}
        \FOR{all $(i,s,h)$} 
        \STATE Forward the utility vector $u^{t-1}_{i,h,s}\leftarrow  \frac{w_{t-1}}{H} Q^{t-1}_{i,h}\pi^{t-1}_{-i,h}(s, \cdot)$ to $\mathfrak{R}_{i,h,s}$. 
        \STATE Update $\pi^t_{i,h}(\cdot \mid s)$ according to $\mathfrak{R}_{i,h,s}$.
        \ENDFOR
        \STATE for all $(i,s,\bolda) \in [n] \times \+S \times \+A$, from $h = H$ to $1$: 
        \begin{align*}
            &Q^t_{i,h}(s,\bolda) \leftarrow (1-\alpha_t)  Q^{t-1}_{i,h}(s,\bolda) \\
            & \quad \quad \quad \quad + \alpha_t (r_h + \-P_h[Q^t_{i,h+1}\pi^t_{h+1}])(s,\bolda).
        \end{align*}
        \ENDFOR
    \end{algorithmic}
    Output $\hpi^T = \hpi^T_1$ as defined in \Cref{alg:output policy}. 
\end{algorithm}

\paragraph{Bounding Correlated Equilibrium Gap by Per-State Regret}
We first show a general result that the output policy $\hpi^T$ of \Cref{alg:main alg} is an approximate correlated equilibrium as long as each player has low \emph{per-state weighted swap regret}. Formally, we define the per-state weighted swap regret (per-state regret for short) up to time $t \ge 1$ with respect to weights $\{\alpha_t^i\}_{i \in [t]}$ as $\reg^{t}_{i,h}(s):=$
\begin{align*}
    \scalebox{0.9}{$\displaystyle\max_{\phi_i} \sum_{j=1}^t \alpha_t^j \InAngles{ Q^j_{i,h}(s, \cdot), ((\phi_i \diamond \pi^j_{i,h}) \odot \pi^j_{-i,h})(\cdot \mid s) - \pi_h^j(\cdot \mid s)}.$ }
\end{align*}
We also define $\reg^{t}_{h}$ as the maximum weighted regret over all players and all states:
\begin{align}\label{eq:weighted regret}
    \reg^{t}_{h} := \max_{s \in \+S} \max_{i \in [n]} \reg^{t}_{i,h}(s).
\end{align}

\begin{theorem}
\label{thm:cegap<= regret}
    Suppose that the per-state regret has upper bounds $\reg_h^t \le \regbar_h^t$ for all $(h, t) \in [H] \times [T]$ where $\regbar_h^t$ is non-increasing in $t$: $\regbar_h^t \ge \regbar_h^{t+1}$. Then the output policy of \Cref{alg:main alg} satisfies
    \[
    \cegap(\hpi^T) \le 2H \cdot \frac{1}{T} \sum_{t=1}^T \max_{h \in [H]} \regbar^t_h.
    \]
    for all $T \ge 2$.
\end{theorem}

\Cref{thm:cegap<= regret} states that $\cegap(\hpi^T)$ can be bounded by the average weighted regret $O(\frac{1}{T} \sum_{t=1}^T \max_{h \in [H]} \regbar^t_h)$. Thus, for any algorithm $\mathfrak{R}$ chosen in the policy update step, as long as the weighted average regret is sublinear, the output policy is an approximate correlated equilibrium. However, we emphasize that minimizing weighted swap regret $\regbar^t_h$ with respect to $\{\alpha_t^i\}_{i \in [t]}$ requires careful design and analysis of the algorithm.  

\paragraph{Proof Overview} For $h \in [H]$, we define the reward difference between policy $\hpi^t_h$ and a best strategy modification over any player $i \in [m]$ and state $s \in \+S$ as:
\[\delta_{h}^t:= \max_{i \in [n]} \max_{s \in \+S}\InParentheses{\max_{\phi_i} V_{i,h}^{(\phi_i \diamond \hpi_{i,h}^t)\odot \hpi_{-i,h}^t}(s) - V_{i,h}^{\hpi^t_{h}}(s)}.\]

In \Cref{lemma:recursively bound gap by regret}, we establish bounds on $\delta^t_h$ using weighted regret  such that
\[
\delta^t_h \le \sum_{j=1}^t \alpha_t^j \delta^j_{h+1} + \reg^t_h.
\]
Then \Cref{thm:cegap<= regret} follows by applying the above inequality recursively to bound   $\cegap(\hpi^T) = \delta^T_1$.

\subsection{Policy Update} 
We now turn to the design of the policy update, with the goal of minimizing the weighted swap regret $\regbar^t_h$.
Each player $i$ maintains a Markov policy $\pi^t_{i,h}(\cdot \mid s)$ for every pair step $h$ and step $s$. During each episode $t\in [T]$, player $i$ updates $\pi^t_{i,h}$ using an online learning algorithm $\mathfrak{R}$. 
Previous works such as~\citep{zhang2022policy,yang2023ot} adopted the optimistic follow-the-regularized-Leader (OFTRL) algorithm~\citep{syrgkanis2015fast} with entropy regularization, which is a no external regret algorithm. However, inspired by recent breakthrough in normal-form games~\citep{anagnostides2022uncoupled}, we select $\mathfrak{R}$ to be a specific no swap regret algorithm as outlined in \Cref{alg:swap regret}. Specifically, \Cref{alg:swap regret} (1) uses the template introduced by \citet{blum2007external}, which constructs a no swap regret algorithm $\mathfrak{R}_{swap}$ from several external regret minimizers $\mathfrak{R}_a$ for each action $a \in \+A_i$; (2) employ \emph{weighted} OFTRL with log barrier regularization for each external regret minimizer $\mathfrak{R}_a$. It is has been shown that with \emph{constant} step size, OFTRL with log barrier regularization guarantees $O(\log T)$ individual swap regret in general-sum games~\citep{anagnostides2022uncoupled}. We extend their analysis to the more challenging Markov games with \emph{decreasing} step size and provide bounds for \emph{weighted swap regret}. A detailed discussion and analysis of \Cref{alg:swap regret} are presented in \Cref{sec:bounding regret}. 

\begin{algorithm}[t]
    \caption{BM-OFTRL-Log-Bar}
    \begin{algorithmic}[1]\label{alg:swap regret}
        \REQUIRE Action set $\+A$, step size $\eta$, weights $\{w_t\}$.
        \STATE \textbf{Initialization:} Initialize $\mathfrak{R}_a$ as an instance of \ref{OFTRL-LogBar} for each $a \in \+A$.
        \FOR{$t = 1, 2, \ldots, T$}
        \STATE Get $x^t_a$ from $\mathfrak{R}_a$ for all $a \in \+A$; Construct a (row) stochastic matrix $M^t \in \mathbb{S}^{|\+A|\times |\+A|}$ where the row that corresponds to $a \in \+A$ is equal to $x^t_a$; Output strategy $x^t \in \Delta_{\+A}$ so that $M^t x^t = x^t$.   
        \STATE Get reward vector $u^t$; Forward $u^t_a:= x^t[a] \cdot u^t $ to $\mathfrak{R}_a$ for each $a \in \+A$.
        \ENDFOR
    \end{algorithmic}
\end{algorithm}

\subsection{Main Results}
We present our main result on the convergence of \Cref{alg:main alg} to correlated equilibrium in multi-player general-sum Markov games. 
\begin{theorem}
\label{thm:main result}
    For an $n$-player general-sum Markov game and any $T \ge 2$, when $\mathfrak{R} = $ \Cref{alg:swap regret} with step size $\eta = \frac{1}{128n\sqrt{H}A_{\max}}$, the output policy $\hpi^T$ of either \Cref{alg:main alg V update} or \Cref{alg:main alg} satisfies 
    \[\cegap(\hpi^T)  \le 8192 H^{3.5} n A_{\max}^3 \cdot \frac{(\log T)^2}{T}.\]
\end{theorem}
We remark again that the previous best rate for finding CE is $\Tilde{O}(T^{-\frac{1}{2}})$ achieved by the V-learning algorithm~\citep{jin2021v, song2021can, mao2023provably},
and \citet{zhang2022policy} provides a faster convergence rate of $\Tilde{O}(T^{-\frac{3}{4}})$ to the weaker notion of CCE. \Cref{thm:main result} improves both results significantly and for the first time, shows a near-optimal $\Tilde{O}(T^{-1})$ convergence rates to CE/CCE in multi-player general-sum Markov games using a single loop and uncoupled policy optimization algorithm\footnote{After our paper is accepted and posted on Arxiv, a concurrent work by \citet{mao2024convergence} was posted on Arxiv which proves a similar result to \Cref{thm:main result}.}. 

\section{Proof of the Main Result}
In this section, we provide a sketch of our analysis along with more explanation on the algorithm design.
We first recall that \Cref{alg:swap regret} applies the template by \citet{blum2007external} that constructs a swap regret minimizer $\mathfrak{R}_{swap}$ from a set of external regret minimizers $\{\mathfrak{R}_a\}_{a \in \+A}$, one for each action $a \in \+A$. The resulting algorithm  $\mathfrak{R}_{swap}$ ensures $\sreg^T = \sum_{a \in \+A} \Reg^T_a$~\citep{blum2007external}.

\subsection{Optimistic Follow the Regularized Leader with Log Barrier Regularization}
\label{sec:bounding regret}
An important component of \Cref{alg:swap regret} is the optimistic follow the regularized leader algorithm with variable step size $\{\eta_t\}$. The optimistic follow the regularized leader \eqref{OFTRL} algorithm~\citep{syrgkanis2015fast} over strategy set $\X$ and with a regularizer $\+R: \X \rightarrow \R$ is defined as follows: $x^0 := \argmin_{x \in \X} \+R(x)$ and for $t \ge 1$, the algorithm updates $x^t$ using step size $\eta_t > 0$ 
\begin{equation}
\label{OFTRL}
\tag{OFTRL}
    \begin{aligned}
    x^{t} = \argmax_{x \in \X} \left \{\eta_t \InAngles{x, m^t + \sum_{\tau=1}^{t-1} u^\tau} - \+R(x)\right\}
    \end{aligned}
\end{equation}
Previous works~\citep{zhang2022policy, yang2023ot} choose $\+R$ to be a strongly convex function such as the entropy regularization. Here we follow \citep{anagnostides2022uncoupled} and let $\+R$ be a \emph{self-concordant barrier}. We first extend the RVU-bound established in \citep{anagnostides2022uncoupled} for \ref{OFTRL}
with a \emph{constant} step size to the case of \emph{variable} step sizes.\footnote{The term RVU is from~\citet{syrgkanis2015fast}, which stands for ``Regret bounded by Variation in Utilities''.} Before stating the result, we first introduce some notations. We assume $\X$ has a nonempty interior $\interior(\X)$. We say $\+R$ is non-degenerate if its Hessian $\nabla^2 \+R(x)$ is positive definite for all $x \in \interior(\X)$. For any vector $u \in \R^d$, the primal \emph{local norm} with respect to $x \in \interior(\X)$ is defined as $\InNorms{u}_{x}:= \sqrt{u^\top \nabla^2 \+R(x) u}$ and the dual norm is defined as $\InNorms{u}_{*,x}:= \sqrt{u^\top (\nabla^{2}\+R(x))^{-1} u}$ when $\+R$ is non-degenerate. We also use $g^t$ to denote the sequence produced by \emph{Be-the-Leader} \eqref{BTL} algorithm. 

\begin{theorem}[RVU for Self-Concordant Barrier with decreasing step size]
\label{thm: RVU-stable}
    Suppose that $\+R$ is a non-degenerate self-concordant barrier function for $\interior(\+X)$ and let $\eta_t > 0$ be such that $\eta_t \InNorms{u^t - m^t}_{*, x^t} \le \frac{1}{2}$ and $\InNorms{\eta_t m^t + (\eta_t - \eta_{t-1}) \sum_{\tau=1}^{t-1} u^\tau }_{*, g^{t-1}} \le \frac{1}{2}$ for all $t \in [T]$. Then, the regret of \ref{OFTRL} with respect to any $x^* \in \interior(\X)$ and under any sequence of utilities $u^1, \ldots, u^T$ can be bounded as
    \begin{align*}
        &\Reg^T(x^*) \le \frac{R(x^*)}{\eta_T} + 2\sum_{t=1}^T\eta_t \InNorms{u^t - m^t}_{*, x^t}^2 \\
        &-\sum_{t=1}^T \InParentheses{\frac{1}{4\eta_t} \InNorms{x^t - g^t}_{x^t}^2 + \frac{1}{4\eta_{t-1}} \InNorms{x^t - g^{t-1}}_{g^{t-1}}^2}.
    \end{align*}
\end{theorem}
\paragraph{Log Barrier Regularization} Now we describe the implementation of \ref{OFTRL} in \Cref{alg:swap regret}. We choose $\+R$ to be the \emph{log barrier} over the simplex $\X = \Delta^d$ defined as $\+R(x) = -\sum_{r=1}^d \log x[r]$. For $t \ge 1$, the step size is $\eta_t = \frac{\eta}{w_t}$  ($w_t$ is defined in \eqref{eq:w_t}) for some $\eta > 0$. In order to minimize the \emph{weighted} regret, we also equip the utilities vectors with weights $\{w_t\}$ so that $u^t = w_t \hu^t$ and $m^t = w_t \hat{m}^t$ with  $\InNorms{\hu^t}_\infty \le 1$ and $\InNorms{\hat{m}^t}_\infty \le 1$. The prediction vector $\hat{m}^t$ is chosen to be $\hu^{t-1}$. We denote the resulting algorithm \ref{OFTRL-LogBar}.
\begin{equation}
\label{OFTRL-LogBar}
\tag{OFTRL-LogBar}
    \begin{aligned}
    x^t =\argmax_{x \in \X} \left\{ \frac{\eta}{w_t} \InAngles{x, w_t \hu^{t-1} + \sum_{\tau=1}^{t-1} w_{\tau} \hu^\tau} - \+R(x) \right\}
    \end{aligned}
\end{equation}

Note that we can not directly apply \Cref{thm: RVU-stable} to \ref{OFTRL-LogBar} since the simplex $\Delta^d$ has empty interior. This can be addressed by a transformation on the relative interior $\relint(\Delta^d)$ which preserves the regret (See Appendix~\ref{app:self-concordant}). The following lemma further verifies that \ref{OFTRL-LogBar} with any $\eta \le \frac{1}{16}$ satisfies the two stability conditions required by \Cref{thm: RVU-stable}.

\begin{lemma}
\label{lemma:stability of eta}
    Let $\eta_t = \frac{\eta}{w_t}$,  $u^t = w_t \hu^t$, and $m^t = w_t \hat{m}^t $ such that $\InNorms{\hu^t}_{\infty}, \InNorms{\hat{m}^t}_\infty \le 1$. Then the iterates of \ref{OFTRL-LogBar} satisfy $\eta_t \InNorms{u^t - m^t}_{*, x^t} \le 2\eta$ and 
    \begin{align*}
        \InNorms{\eta_t m^t + (\eta_t - \eta_{t-1})\sum_{\tau=1}^{t-1} u^\tau }_{*, g^{t-1}} \le 3\eta.
    \end{align*}
\end{lemma}
\begin{proof}
    By definition, for any vector $u \in \R^d$ and $x \in \interior(\Delta^d)$, it holds that $\InNorms{u}_{*, x} \le \InNorms{u}_\infty$. Then we have
    \begin{align*}
        \eta_t\InNorms{u^t - m^t}_{*, x^t} =  \eta \InNorms{\hu^t - \hat{m}^t}_{*, x^t} &\le  \eta \InNorms{\hu^t - \hat{m}^t}_{\infty} \le 2\eta.
    \end{align*}
    Using properties of $\{w_t\}$ (\Cref{lemma:alpha w properties}), we have
    \begin{align*}
        &\InNorms{\eta_t m^t + (\eta_t - \eta_{t-1}) \sum_{\tau=1}^{t-1} u^\tau }_{*, g^{t-1}} \\
        &\le \eta \InNorms{\hat{m}^t}_{\infty} + \eta \InNorms{  \InParentheses{ \frac{1}{w_{t-1}} - \frac{1}{w_t}} \sum_{\tau=1}^{t-1} w^i \hu^\tau}_{\infty} \\
        &\le \eta + \eta \InParentheses{ \frac{1}{w_{t-1}} - \frac{1}{w_t}} \sum_{\tau=1}^{t-1} w^i  \\
        &\le \eta + \eta \cdot \frac{H+1}{H} \le 3\eta. \qedhere
    \end{align*}
\end{proof}
Combining \Cref{thm: RVU-stable} and \Cref{lemma:stability of eta} with additional analysis, we have the following RVU bound for \ref{OFTRL-LogBar}.

\begin{corollary}
\label{corollary: RVU-logbar}
    Let $\eta\le \frac{1}{16}$. Then, the regret of \ref{OFTRL-LogBar} under any sequence of utilities $u^1, \ldots, u^T$ can be bounded as
    \begin{align*}
        &\Reg^T(x^*) \le \frac{R(x^*)}{\eta_T} + 2\sum_{t=1}^T\eta_t \InNorms{u^t - m^t}_{*, x^t}^2\\
        &- \sum_{t=1}^T \frac{1}{16\eta_{t-1}} \InNorms{x^t - x^{t-1}}_{x^{t-1}}^2, 
    \end{align*}
    for any $x^* \in \relint(\Delta^d)$, where $\InNorms{x^t - x^{t-1}}^2_{x^{t-1}}:=\sum_{r=1}^{d}(\frac{x^t[r] - x^{t-1}[r]}{x^{t-1}[r]})^2$.
\end{corollary}

\paragraph{Swap Regret} 
Applying \Cref{corollary: RVU-logbar} to each $\mathfrak{R}_a$, then the swap regret $\sreg^T = \sum_{a \in \+A} \Reg^T_a$ of \Cref{alg:swap regret}  is upper bounded by 
\begin{align*}
    &\frac{2|\+A|^2\log T}{\eta_T} + 2 \sum_{t=1}^T \eta_t \InNorms{u^t - m^t}_\infty^2  \\
    &- \sum_{t=1}^T\sum_{a \in \+A} \frac{\InNorms{x^t_a- x^{t-1}_a}^2_{x^{t-1}_a}}{16\eta_{t-1}},
\end{align*}
where we get a $\log T$ factor due to the diameter of the log barrier regularizer. Following techniques developed in \citep{anagnostides2022uncoupled}, we show a key lemma that lower bounds $\sum_{a \in \+A} \frac{1}{16\eta_{t-1}}\InNorms{x^t_a- x^{t-1}_a}^2_{x^{t-1}_a}$ by movement in the output strategies $\InNorms{x^t - x^{t-1}}_1^2$.
\begin{lemma}
\label{lemma: l1norm}
    Suppose $\eta < \frac{1}{28 |\+A|}$. Then the iterates of \Cref{alg:swap regret} satisfies for any $t \in [T]$, 
    \begin{align*}
        \InNorms{x^t - x^{t-1}}_1^2 \le 64 |\+A| \sum_{a \in \+A} \InNorms{x^t_a - x^{t-1}_a}^2_{x^{t-1}_a}.
    \end{align*}
\end{lemma}
Combining the above then gives a RVU bound for swap regret of \Cref{alg:swap regret} with variable step size.
\begin{theorem}[RVU for swap regret]
\label{thm:RVU for swap regret}
    Let $\eta < \frac{1}{28 |\+A|}$. Then for any $t \in [T]$, the swap regret of \Cref{alg:swap regret} is at most 
    \begin{align*}
        \frac{2|\+A|^2 \log T}{\eta_T} + \sum_{t=1}^T 4\eta_t \InNorms{u^t - m^t}_\infty^2 - \sum_{t=1}^T \frac{\InNorms{x^t - x^{t-1}}_1^2}{1024|\+A|\eta_{t-1}} .
    \end{align*}
\end{theorem}
\subsection{Bounding Per-State Regret} 
In this subsection, we prove upper bounds for $\reg_{i,h}^t(s)$, the per-state regret for player $i \in [n]$ and any $(h,s,t) \in [H] \times [\+S] \times [T]$. For simplicity of notation, throughout this subsection, we fix $(i, h, s)$ and omit the subscripts $(h, s)$ within the policies and $Q$-functions, i.e., $\pi_{i,h}^t(\cdot \mid s)$ will be abbreviated as $\pi_{i}^t$ and $Q_{i,h}^t(s, \cdot)$ will be abbreviated as $Q^t_i$. We also overload $T$ be any iteration $T \ge 1$.

Recall the policy update step in \Cref{alg:main alg} where we feed $u^t_i = w_t \frac{1}{H} (Q^t_i)^\top \pi^t_{-i}$ to BM-OFTRL-LogBar (\Cref{alg:swap regret}). Thus we can relate $\reg_{i,h}^T(s)$ to the regret incurred by OFTRL-Log-Bar for any $T \ge 1$ as follows:
\begin{align}
    \reg_{i,h}^{T}(s) &= \max_{\phi_i} \sum_{t=1}^{T} \alpha_{T}^t \InAngles{\phi_i \diamond \pi^t_{i}- \pi_i^t, Q^t_i \pi^t_{-i}} \nonumber \\
    &=H \alpha_{T}^1 \cdot \underbrace{\max_{\phi_i} \sum_{t=1}^{T}  \InAngles{\phi_i \diamond \pi^t_{i}- \pi_i^t, u^t_i}}_{\sreg_i^{T}}, \label{eq:weighted to unweighted}
\end{align}
where the second equality holds since $ \alpha^t_{T} = \alpha_{T}^1 w_t$ (defined in \Cref{eq:alpha_t^i} and \Cref{eq:w_t}).
Now we apply \Cref{thm:RVU for swap regret} and obtain that for any $T \ge 1$, 
\begin{align*}
    &\sreg^{T}_i = \max_{\phi_i} \sum_{t=1}^{T} \InAngles{\phi_i \diamond \pi^t_{i}- \pi_i^t, u^t_i}\\
    &\le \frac{2A_i^2 \log T}{\eta_{T}} +  \sum_{t=1}^{T} 2\eta_t \InNorms{u^t_i -u^{t-1}_i}^2_\infty - \sum_{t=1}^{T} \frac{\InNorms{\pi_i^t - \pi_i^{t-1}}_1^2}{1024A_i\eta_{t-1}}  \\
    &\le \frac{2w_{T} A_i^2 \log T}{\eta} + \sum_{t=1}^{T} \frac{2\eta w_t}{H^2} \InNorms{Q^t_i \pi^t_{-i} - Q^{t-1}_i \pi^{t-1}_{-i}}^2_\infty \\
    & \quad \quad \quad -\sum_{t=1}^{T}\frac{w_t}{1024 A_i \eta H}  \InNorms{\pi_i^t - \pi_i^{t-1}}_1^2,\\
\end{align*}
where we use $\eta_t =\frac{\eta}{w_t}$, $u^t_i = w_t \frac{1}{H} (Q^t_i)^\top \pi^t_{-i}$ and $w_{t-1} = \frac{w_t (t-1)}{H+t-1} \ge \frac{w_t}{H}$ in the last equality.
Using the fact that $\InNorms{Q^t_i \pi^t_{-i} - Q^{t-1}_i \pi^{t-1}_{-i}}^2_\infty \le 2\InNorms{Q^t_i - Q^{t-1}_i}^2_\infty + 2\InNorms{Q^t_i(\pi^t_{-i}-\pi^{t-1}_{-i})}^2_\infty$ as well as  $\InNorms{Q^t_i - Q^{t-1}_i}^2_\infty \le \alpha_t^2 H^2$ and $\InNorms{Q^t_i(\pi^t_{-i}-\pi^{t-1}_{-i})}^2_\infty \le H^2 \InNorms{\pi^t_{-i} - \pi^{t-1}_{-i}}^2_1$, we can further upper bound $\sreg^{T}_i$ as 
\begin{align}
    &\frac{2w_{T} A_i^2 \log T}{\eta} + 4\eta\underbrace{\sum_{t=1}^{T} w_t (\alpha_t)^2}_{\textbf{I}} + 4\eta \underbrace{ \sum_{t=1}^{T} w_t  \InNorms{\pi^t_{-i} - \pi^{t-1}_{-i}}^2_1}_{\textbf{II}} \nonumber \\
    &- \sum_{t=1}^{T}\frac{w_t}{1024A_i\eta H} \InNorms{\pi_i^t - \pi_i^{t-1}}_1^2. \label{eq:individual regret}
\end{align}
Now we focus on upper bounding term \textbf{I} and \textbf{II}. From \Cref{lemma:alpha w properties}, we know $\sum_{t=1}^{T} \alpha_{T}^t (\alpha_t)^2 \le \frac{4H}{T}$. This further implies $\textbf{I} \le \frac{4H}{\alpha_{T}^1 T}$ since $\alpha_{T}^t = \alpha_{T}^1 w_t$. 

For term \textbf{II}, we have
\begin{align*}
    \textbf{II} &= \sum_{t=1}^{T} w_t  \InNorms{\pi^t_{-i} - \pi^{t-1}_{-i}}^2_1 \le \sum_{t=1}^{T} w_t \InParentheses{ \sum_{j \ne i} \InNorms{\pi^t_{j} - \pi^{t-1}_j}_1}^2 \\
    &\le (n-1) \sum_{t=1}^{T} w_t\sum_{j\ne i} \InNorms{\pi^t_{j} - \pi^{t-1}_j}_1^2, 
\end{align*}
where the first inequality holds since the total variational distance between two product distribution is bounded by the sum of total variational distance between each marginal distribution. 

Then the total swap regret $\sum_{i=1}^n \sreg^T_i$ can be upper bounded by
\begin{align*}
    & \sum_{i=1}^n \sreg^T_i \le \frac{2 w_{T}n A^2_{\max} \log T}{\eta} + 4\eta n \sum_{t=1}^{T} w_t (\alpha_t)^2 \\
    &+ \sum_{j=1}^n \sum_{t=1}^{T} \InParentheses{4\eta w_t n^2 - \frac{w_t}{1024 A_j \eta H} } \InNorms{\pi^t_j - \pi^{t-1}_j}_1^2 \\
    &\le \frac{2 w_{T} n A^2_{\max}\log T}{\eta} + 4\eta n \sum_{t=1}^{T} w_t (\alpha_t)^2  \\
    &-  4\eta n^2 \sum_{j=1}^n \sum_{t=1}^{T} w_t\InNorms{\pi^t_j - \pi^{t-1}_j}_1^2,
\end{align*}
since $\eta = \frac{1}{128n\sqrt{H}A_{\max}}$. Since the swap regret is non-negative, the above inequality implies
\begin{align*}
    &\textbf{II} \le  n \sum_{j=1}^n \sum_{t=1}^{T} w_t \InNorms{\pi^t_j - \pi^{t-1}_j}_1^2 \\
    &\le  8192 H n^2 A_{\max}^4  w_{T} \log T + \underbrace{ \sum_{t=1}^{T} w_t (\alpha_t)^2}_{=\textbf{I}}.
\end{align*}
Now we can plug the above bounds on terms \textbf{I} and \textbf{II} into the individual regret bound in \Cref{eq:individual regret} and multiply $H \alpha_{T}^1$ (\Cref{eq:weighted to unweighted}) to bound $\reg^{T}_{i,h}(s)$. Since $\eta = \frac{1}{128n\sqrt{H}A_{\max}}$ and $\alpha_{T}^1 w_{T} = \alpha_{T} \le \frac{2H}{T}$, we finally get
\begin{align}
    \reg_{i,h}^{T}(s) &= H\alpha_{T}^1 \cdot \sreg^{T}_i \nonumber \\
    &\le \frac{2H A_{\max}^2 \alpha_{T}\log T}{\eta} + 4\eta H \alpha_{T}^1 \cdot  ( \textbf{I} + \textbf{II}) \nonumber \\
    &\le 2048 n H^{5/2} A_{\max}^3 \frac{\log T}{T}
    .  \label{eq:bound on weighted regret}
\end{align}
Since the above holds for all $(i, s, h) \in [n] \times [\+S] \times [H]$ and $T \ge 1$, we conclude that the maximum weighted regret over all players,  all states, and all steps is  $\max_{h \in [H]} \regbar^t_h \le O(\frac{\log t}{t})$.  

\paragraph{Proof of \Cref{thm:main result}} Combining \Cref{thm:cegap<= regret} and the weighted regret upper bound in \Cref{eq:bound on weighted regret}, we conclude that
\begin{align*}
    \cegap(\hpi^T) &\le 2H \cdot \frac{1}{T} \sum_{t=1}^T \max_{h \in [H]} \regbar^t_h\\
    &\le 4096 H^{3.5} n A_{\max}^3 \cdot \frac{1}{T} \sum_{t=1}^T \frac{\log t}{t} \\
    &\le 8192 H^{3.5} n A_{\max}^3 \cdot \frac{(\log T)^2}{T}.
\end{align*}
This completes the proof.

\section{Conclusion and Future Directions}
In this work, we propose a policy optimization algorithm with near-optimal $\Tilde{O}(T^{-1})$ convergence rate to correlated equilibrium in general-sum Markov games. Our result improves the results and answers the open questions in previous works~\citep{zhang2022policy,yang2023ot}. 

A natural future direction is to further improve the convergence rates with respect to the number of iterations $T$.  We remark that shaving the $\poly \log T$ factors is challenging even in normal-form games. Other directions include improving the dependence on the horizon $H$ and size of action set $A_{\max}$, and generalizing our results in the setting with oracle access to reward/transition to the sample-based setting where all game parameters are unknown and have to be learned from iteractions.

\paragraph{Acknowledgements} HL is supported by NSF award IIS-1943607 and a Google Research Scholar Award. YC and WZ are supported by the NSF Award CCF-1942583 (CAREER).

\bibliographystyle{apalike}
\bibliography{ref.bib}

\section*{Checklist}

The checklist follows the references. For each question, choose your answer from the three possible options: Yes, No, Not Applicable.  You are encouraged to include a justification to your answer, either by referencing the appropriate section of your paper or providing a brief inline description (1-2 sentences). 
Please do not modify the questions.  Note that the Checklist section does not count towards the page limit. Not including the checklist in the first submission won't result in desk rejection, although in such case we will ask you to upload it during the author response period and include it in camera ready (if accepted).

\textbf{In your paper, please delete this instructions block and only keep the Checklist section heading above along with the questions/answers below.}

 \begin{enumerate}

 \item For all models and algorithms presented, check if you include:
 \begin{enumerate}
   \item A clear description of the mathematical setting, assumptions, algorithm, and/or model. [Yes]
   \item An analysis of the properties and complexity (time, space, sample size) of any algorithm. [Yes]
   \item (Optional) Anonymized source code, with specification of all dependencies, including external libraries. [Not Applicable]
 \end{enumerate}

 \item For any theoretical claim, check if you include:
 \begin{enumerate}
   \item Statements of the full set of assumptions of all theoretical results. [Yes]
   \item Complete proofs of all theoretical results. [Yes]
   \item Clear explanations of any assumptions. [Yes]     
 \end{enumerate}

 \item For all figures and tables that present empirical results, check if you include:
 \begin{enumerate}
   \item The code, data, and instructions needed to reproduce the main experimental results (either in the supplemental material or as a URL). [Not Applicable]
   \item All the training details (e.g., data splits, hyperparameters, how they were chosen). [Yes/No/Not Applicable]
         \item A clear definition of the specific measure or statistics and error bars (e.g., with respect to the random seed after running experiments multiple times). [Not Applicable]
         \item A description of the computing infrastructure used. (e.g., type of GPUs, internal cluster, or cloud provider). [Not Applicable]
 \end{enumerate}

 \item If you are using existing assets (e.g., code, data, models) or curating/releasing new assets, check if you include:
 \begin{enumerate}
   \item Citations of the creator If your work uses existing assets. [Not Applicable]
   \item The license information of the assets, if applicable. [Not Applicable]
   \item New assets either in the supplemental material or as a URL, if applicable. [Not Applicable]
   \item Information about consent from data providers/curators. [Not Applicable]
   \item Discussion of sensible content if applicable, e.g., personally identifiable information or offensive content. [Not Applicable]
 \end{enumerate}

 \item If you used crowdsourcing or conducted research with human subjects, check if you include:
 \begin{enumerate}
   \item The full text of instructions given to participants and screenshots. [Not Applicable]
   \item Descriptions of potential participant risks, with links to Institutional Review Board (IRB) approvals if applicable. [Not Applicable]
   \item The estimated hourly wage paid to participants and the total amount spent on participant compensation. [Not Applicable]
 \end{enumerate}

 \end{enumerate}

\newpage
\onecolumn
\aistatstitle{Supplementary Materials for Near-Optimal Policy Optimization for Correlated Equilibrium in
General-Sum Markov Games}
\appendix

\thispagestyle{empty}
\setcounter{tocdepth}{3}
\tableofcontents
\thispagestyle{empty}

\section{Properties of $\alpha_t^i$ and $w_i$}
We present several useful properties of the sequence $\{\alpha_t^i\}_{t\ge 1, 1 \le i \le t}$ and $\{w_t\}_{t\ge 1}$ in the following lemmas that are known in previous works~\citep{jin2018q, zhang2022policy}. We first recall that $\alpha_t = \frac{H+1}{H+t}$ for all $t \ge 1$. The definitions of $\{\alpha_t^i\}_{t\ge 1, 1 \le i \le t}$ and $\{w_t\}_{t\ge 1}$ are 
\[
\alpha_t^t = \alpha_t, \quad \alpha_t^i = \alpha_i \Pi_{j=i+1}^t (1-\alpha_j),  \forall i \le t-1,
\]
and 
\[
w_t = \frac{\alpha_t^t}{\alpha_t^1}.
\]

\begin{lemma}[\citep{jin2018q}]
\label{lemma:alpha property}
    The sequence $\{\alpha_t^i\}_{t\ge 1, 1 \le i \le t}$ satisfies the following:
    \[
    \sum_{t=i}^\infty \alpha_t^i = 1 + \frac{1}{H}, \forall i \ge 1.
    \]
\end{lemma}

\begin{lemma}[\citep{zhang2022policy}] 
\label{lemma:alpha w properties}
The sequence $\{\alpha_t^i\}_{t\ge 1, 1 \le i \le t}$ and $\{w_t\}_{t\ge 1}$ satisfies the following:
    \begin{itemize}
        \item[1.] $\sum_{t=1}^T \alpha_T^t (\alpha_t)^2 \le \frac{4H}{T} $.
        \item[2.] $w^t = \frac{\alpha_T^t}{\alpha_T^1}$ for all $T \ge t$.
        \item[3.] $(\frac{1}{w_{t-1}}-\frac{1}{w_t})\sum_{i=1}^{t-1} w^i =\frac{H+1}{H}$.
        \item[4.] Given a sequence $\{\Delta^t_h\}_{h,t}$ defined by \begin{align*}
            \begin{cases}
                \Delta^t_h = \sum_{i=1}^t\alpha_t^i \Delta^{i}_{h+1} + \beta_t,\\
                \Delta^t_{H+1} = 0, \forall t,
            \end{cases}
        \end{align*}
        where $\{\beta_t\}$ is non-increasing in $t$, then $\Delta^{t+1}_h \le \Delta^{t}_h$ for all $t \ge 1$ and $h \in [H+1]$.
    \end{itemize}
\end{lemma}

\section{Missing Proofs in \Cref{sec:value update}}

\subsection{Equivalence between $V$ update and $Q$ update}

In this section, we prove the equivalence between \Cref{alg:main alg V update}  and \Cref{alg:main alg}. 
\begin{proposition}[Equivalence between $V$ update and $Q$ update]
\label{prop:V=Q}
    \Cref{alg:main alg V update} and \Cref{alg:main alg} are equivalent in the sense that they produce the same sequence of policies $\{\pi^t_{h}\}(s, \cdot)$.
\end{proposition}
\begin{proof}
    It suffices to prove that the $Q$ value in \Cref{alg:main alg} and $V$ value in \Cref{alg:main alg V update} satisfies the following: for any $(i,s,h, \bolda)$ and $t \in [T]$, 
    \begin{equation}
    \label{eq:Q=V}
        Q^t_{i,h}(s, \bolda) = \InBrackets{r_h + \-P_h V^t_{i,h+1}}(s, \bolda).
    \end{equation}
    Note that the above holds for $t = 0$ according to the initialization step in both algorithms. Since $\alpha_t = 1$ and $Q^1_{i, H+1}(s,\bolda) = V^t_{i, H+1}(s) = 0$, we can also verify by induction that \Cref{eq:Q=V} holds for $t = 1$:
    \[
    Q^1_{i,h}(s, \bolda) = \InBrackets{r_h + \-P_h[Q^1_{i, h+1} \pi^1_{h+1}] }(s, \bolda) = \InBrackets{r_h + \-P_h V^t_{i,h+1}}(s, \bolda).
    \]
    Moreover, it is easy to see $Q^t_{i,H}(s, \bolda) = r_H(s, \bolda)$ for all $t \ge 1$. Now we conduct induction on both $t$ and $h$. We assume \Cref{eq:Q=V} holds for $(t-1, h)$ and $(t, h+1)$, then for $(t, h)$, we have
    \begin{align*}
        Q^t_{i,h}(s, \bolda) &= (1-\alpha_t)  Q^{t-1}_{i,h}(s,\bolda) + \alpha_t (r_h + \-P_h[Q^t_{i,h+1}\pi^t_{h+1}])(s,\bolda) \\
        &= (1- \alpha_t) \InBrackets{r_h + \-P_{h} V^{t-1}_{i,h+1}}(s, \bolda) + \alpha_t (r_h + \-P_h[(r_{h+1} + \-P_{h+1}V^t_{i, h+2})\pi^t_{h+1}])(s,\bolda) \tag{by induction hypothesis} \\
        &= \InBrackets{r_h +  \-P_h \InParentheses{  (1-\alpha_t) V^{t-1}_{i,h+1} + \alpha_t ( r_{h+1} + \-P_{h+1} V^{t}_{i,h+2})  \pi^t_{h+1} }}(s, \bolda)\\
        &= \InBrackets{r_h + \-P_h V^t_{i,h+1}}(s, \bolda). \tag{by update rule of $V^t_{i,h+1}(s)$ in \Cref{alg:main alg V update} }
    \end{align*}
    This completes the proof.
\end{proof}

\subsection{Proof of \Cref{thm:cegap<= regret}}
We need the following two technical lemmas in the proof of \Cref{thm:cegap<= regret}. In \Cref{lemma:V=Vhat}, we show that the value function $V^t_{i,h}(s)$ maintained in \Cref{alg:main alg V update} equals to $V^{\hpi^t_h}_{i,h}(s)$ where the policy $\hpi^t_h$ is defined in \Cref{alg:output policy}. In \Cref{lemma:recursively bound gap by regret}, we prove a recursive inequality that bounds the $\cegap$ of $\hpi^t_h$ by weighted regret (as defined in \eqref{eq:weighted regret}).
\begin{lemma}
\label{lemma:V=Vhat}
    For all player $i$ and $(h,s) \in [H+1] \times \+S$ and $V^t_{i,h}(s)$ being the $V$ values maintained in \Cref{alg:main alg V update}, it holds that
    \begin{itemize}
        \item[1.] $V^t_{i,h}(s) = \sum_{j=1}^t \alpha_t^j \InBrackets{(r_h + \-P_h V_{i,h+1}^j)\pi^j_h}(s)$.
        \item[2.] $V^t_{i,h}(s) = V^{\hpi^t_h}_{i,h}(s)$ while $\hpi^t_h$ are defined in \Cref{alg:output policy}.
    \end{itemize}
\end{lemma}
\begin{proof}
    Recall the update rule of $V$ value in \Cref{alg:main alg V update}:
    \[
         V^t_{i,h}(s) = (1 - \alpha_t) V^{t-1}_{i,h}(s) + \alpha_t \InBrackets{(r_h + \mathbb{P}_h V^t_{i,h+1})\pi^t_{h} }(s).
    \]
    Then the first claim holds by applying the above recursively for $j \in [t]$.

    Given the first claim and the equivalence between \Cref{alg:main alg V update} and \Cref{alg:main alg}, the second claim follows from \citep[Lemma G.1]{zhang2022policy}.
\end{proof}

\begin{lemma}
\label{lemma:recursively bound gap by regret}
    For the policy $\hpi^t_h$ defined in ..., we have for all $(i, h, t) \in [n] \times [H] \times [T]$ that 
    \[
    \max_{s \in \+S} \InParentheses{\max_{\phi_i} V_{i,h}^{\phi_i \diamond \hpi_{i,h}^t, \hpi_{-i,h}^t}(s) - V_{i,h}^{\hpi^t_{h}}(s)} \le \sum_{j=1}^t \alpha_t^j \max_{s' \in \+S} \InParentheses{\max_{\phi_i'} V_{i, h+1}^{\phi_{i'} \diamond \hpi_{i,h+1}^{j} \times \hpi_{-i, h+1}^j }(s') - V_{i,h+1}^{\hpi_{i,h+1}^j}(s') } + \reg_{h}^t.
    \]
\end{lemma}

\begin{proof}
    Fix $(i, h, t) \in [n] \times [H] \times [T]$. We have for all state $s \in \+S$ that 
    \begin{align*}
        &\max_{\phi_i} V_{i,h}^{\phi_i \diamond \hpi_{i,h}^t, \hpi_{-i,h}^t}(s) - V_{i,h}^{\hpi^t_{h}}(s) \\
        &\le \max_{\phi_i} \InAngles{ (\phi_i \diamond \pi^j_{i,h})(\cdot \mid s), \sum_{j=1}^t \alpha_t^j \InBrackets{(r_h + \max_{\phi_{i'}} \-P_h V_{i, h+1}^{ \phi_{i'} \diamond \hpi_{i,h+1}^{j} \times \hpi_{-i, h+1}^j  }) \pi_{-i,h}^j}(s, \cdot)   } \\
        & \quad \quad \quad \quad - \sum_{j=1}^t \alpha_t^j \InAngles{ \pi_{i,h}^j(\cdot \mid s),  \InBrackets{(r_h + \-P_h V_{i, h+1}^{ \hpi_{i, h+1}^j  })\pi_{-i,h}^j}(s, \cdot)} \\
        &\le \sum_{j=1}^t \alpha_t^j \max_{s' \in \+S} \InParentheses{\max_{\phi_i'} V_{i, h+1}^{\phi_{i'} \diamond \hpi_{i,h+1}^{j} \times \hpi_{-i, h+1}^j }(s') - V_{i,h+1}^{\hpi_{i,h+1}^j}(s') }\\
        &\quad \quad \quad \quad + \max_{\phi_i} \sum_{j=1}^t \InAngles{(\phi_i \diamond \pi^j_{i,h})(\cdot \mid s) - \pi_{i,h}^j(\cdot \mid s), \InBrackets{(r_h + \-P_h V_{i, h+1}^{ \hpi_{i, h+1}^j  })\pi_{-i,h}^j}(s, \cdot) }\\
        &= \sum_{j=1}^t \alpha_t^j \max_{s' \in \+S} \InParentheses{\max_{\phi_i'} V_{i, h+1}^{\phi_{i'} \diamond \hpi_{i,h+1}^{j} \times \hpi_{-i, h+1}^j }(s') - V_{i,h+1}^{\hpi_{i,h+1}^j}(s') }\\
        &\quad \quad \quad \quad + \underbrace{\max_{\phi_i} \sum_{j=1}^t \InAngles{(\phi_i \diamond \pi^j_{i,h})(\cdot \mid s) - \pi_{i,h}^j(\cdot \mid s), \InBrackets{(r_h + \-P_h V_{i, h+1}^{j})\pi_{-i,h}^j}(s, \cdot) }}_{\reg_{i,h}^t(s)}\\
        &\le \sum_{j=1}^t \alpha_t^j \max_{s' \in \+S} \InParentheses{\max_{\phi_i'} V_{i, h+1}^{\phi_{i'} \diamond \hpi_{i,h+1}^{j} \times \hpi_{-i, h+1}^j }(s') - V_{i,h+1}^{\hpi_{i,h+1}^j}(s') } + \reg_{h}^t,
    \end{align*}
    where in the first inequality we use the definition of the policy $\hpi^t_h$ which plays $\pi^j_h$ with probability $\alpha_t^j$ for $ j \in [t]$ and then plays $\hpi^j_{h+1}$ afterwards; in the second inequality, we replace $\max_{\phi_i'} V_{i, h+1}^{\phi_{i'} \diamond \hpi_{i,h+1}^{j} \times \hpi_{-i, h+1}^j }(s')$ with $V_{i,h+1}^{\hpi_{i,h+1}^j}(s')$ and pay the difference; in the equality, we use $V_{i,h+1}^{\hpi_{i,h+1}^j} = V_{i, h+1}^{j}$ by \Cref{lemma:V=Vhat}; in the last inequality we use the definition of weighted regret (\Cref{eq:weighted regret}).
\end{proof}

\paragraph{Proof of \Cref{thm:cegap<= regret}}
    Define $\delta_{h}^t:= \max_{i \in [n]} \max_{s \in \+S}\InParentheses{\max_{\phi_i} V_{i,h}^{\phi_i \diamond \hpi_{i,h}^t, \hpi_{-i,h}^t}(s) - V_{i,h}^{\hpi^t_{h}}(s)}$. 
    Then by \Cref{lemma:recursively bound gap by regret} we have 
    \[
    \delta^t_h \le \sum_{j=1}^t \alpha_t^j \delta_{h+1}^j + \reg_h^t. 
    \]
    Now let us define an auxiliary sequence $\{\Delta_h^t\}_{h, t}$ such that $\Delta_{H+1}^t = 0$ for all $t$ and 
    \begin{align*}
        \Delta^t_h \le \sum_{j=1}^t \alpha_t^j \Delta_{h+1}^j + \regbar_h^t.
    \end{align*}
    Note that $\Delta_h^t \ge \delta^t_h$ for all $(h,t)$ and $\Delta^{t+1}_h \le \Delta_h^t$ (by \Cref{lemma:alpha w properties}). It implies that 
    \begin{align*}
        \Delta_h^t \le  \frac{1}{t}\sum_{j=1}^t \Delta_h^j  &\le \frac{1}{t}\sum_{j=1}^t \sum_{k=1}^j  \alpha_j^k \Delta_{h+1}^k + \frac{1}{t} \sum_{j=1}^t \regbar_h^j \\
        & \le \frac{1}{t} \sum_{k=1}^t (\sum_{j=k}^t \alpha_j^k )\Delta_{h+1}^k + \frac{1}{t} \sum_{j=1}^t \regbar_h^j \\
        & \le (1 + \frac{1}{H}) \cdot \frac{1}{t} \sum_{j=1}^t \Delta_{h+1}^j + \frac{1}{t} \sum_{j=1}^t \regbar_h^j \tag{\Cref{lemma:alpha property}}\\
        & \le (1 + \frac{1}{H})^2 \cdot \frac{1}{t} \sum_{j=1}^t \Delta_{h+2}^j + (1 + \frac{1}{H}) \cdot \frac{1}{t} \sum_{j=1}^t \regbar_{h+1}^j + \frac{1}{t} \sum_{j=1}^t \regbar_h^j \\
        & \le \ldots \\
        & \le \InParentheses{\sum_{h' = h}^H (1+\frac{1}{H})^{h'- h}} \cdot \frac{1}{t} \sum_{j=1}^t \max_{h' \in [H]} \regbar_{h'}^j \\
        & \le (e-1) H \cdot \frac{1}{t} \sum_{j=1}^t \max_{h' \in [H]} \regbar^j_{h'} \\
        & \le 2 H \cdot \frac{1}{t} \sum_{j=1}^t \max_{h' \in [H]} \regbar^j_{h'}.
    \end{align*}
    Thus $\cegap(\pi^T) = \cegap(\pi^T_1) = \delta^T_1 \le \Delta^T_1 \le 2 H \cdot \frac{1}{T} \sum_{t=1}^T \max_{h \in [H]} \regbar^t_{h}$. This completes the proof.

\section{Background on Self-Concordant Barriers}
\label{app:self-concordant}
In this section, we present the necessary background on self-concordant barriers and properties of the log barrier that we use in \Cref{alg:swap regret}. We refer the readers to \Citep{nesterov2003introductory,nemirovski2004interior} for a more comprehensive overview of self-concordant barriers. 

\subsection{Self-Concordant Functions}
\begin{definition}[Self-Concordant Function]
    Let $Q \subseteq \R^d$ be a nonempty open and convex set. A convex function $f : Q \rightarrow \R$ in $\+C^3$ is called \emph{self-concordant} on $Q$ if it satisfies the following two properties:
    \begin{itemize}
        \item[1.]  For every sequence $\{x_i \in Q\}_{i=1}^\infty$ converging to a boundary point of $Q$ as $i \rightarrow \infty$ it holds that $f(x_i) \rightarrow \infty$.
        \item[2.] The functions $f$ satisfies the inequality
        \[
        |D^3f(x)[u,u,u]| \le 2 (D^2 f(x)[u,u])^{3/2},
        \]
        for all $x \in Q$ and $u \in \R^d$. Here $D^kf(x)[u_1, \ldots, u_k]$ denotes the $k$-th-order differential of $f$ at point $x$ along the directions $u_1, \ldots, u_k$.
    \end{itemize}
\end{definition}
As an example, the log barrier for the non-negative ray, i.e., $f : (0,\infty) \owns x \rightarrow -\log x$, is self-concordant. In the following, we assume $f$ is \emph{non-degenerate}, in the sense that the Hessian $\nabla^2 f(x)$ is positive definite for all $x \in \mathrm{dom} f$. In this context, for any vector $u \in \R^d$, we can define the primal \emph{local norm} with respect to $x \in \interior(\X)$ as $\InNorms{u}_{x}:= \sqrt{u^\top \nabla^2 \+R(x) u}$ and the dual norm as $\InNorms{u}_{*,x}:= \sqrt{u^\top (\nabla^{2}\+R(x))^{-1} u}$. We present some useful properties of self-concordant functions below.

\begin{lemma}[\citep{nesterov2003introductory}]
    Let $f$ be a self-concordant function. Then, for any $x, x' \in \mathrm{dom} f$,
    \[
    f(x') \ge f(x) + \InAngles{\nabla f(x), x' - x} + \omega(\InNorms{x' -x}_x),
    \]
    where $\omega(s):= s - \log(1+s)$.
\end{lemma}

\begin{fact}[\citep{anagnostides2022uncoupled}]
\label{fact:omega(s)}
    Let $\omega(s) = s - \log(1+s)$. Then, 
    \[
    \omega(s) \ge \frac{s^2}{2(1+s)}.
    \]
    In particular, for $s \in [0,1]$, it holds that $\omega(s) \ge \frac{s^2}{4}$.
\end{fact}

\begin{lemma}[\citep{nesterov2003introductory}]
\label{lemma:unique solution self-concordant}
    Let $f$ be a self-concordant function such that $\InNorms{\nabla f(x)}_{*, x} < 1$ for some $x \in \mathrm{dom} f$. Then the optimization problem 
    \[
    \min_{x \in \mathrm{dom} f} f(x)
    \]
    has a unique solution.
\end{lemma}

\begin{lemma}[\citep{nemirovski2004interior}]
\label{lemma:local norm}
    Let $\X$ be a convex and compact set with nonempty interior and $f: \interior(\X) \rightarrow \R$ be a self-concordant function with $x^* = \argmin_x f(x)$. Then for any $x \in \mathrm{dom} f$ such that $\InNorms{\nabla f(x)}_{*, x} \le \frac{1}{2}$, it holds that 
    \[
    \InNorms{x - x^*}_x \le 2 \InNorms{\nabla f(x)}_{*, x}, \quad  \InNorms{x - x^*}_{x^*} \le 2 \InNorms{\nabla f(x)}_{*, x}.
    \]
\end{lemma}
\subsection{Self-Concordant Barriers and the Log Barrier}
\begin{definition}[Self-Concordant Barrier]
    Let $\X \subseteq \R^d$ be a convex and compact set with nonempty interior $\interior(\X)$. A function $f: \interior(\X) \rightarrow \R$ is called a $\theta$-self-concordant barrier for $\X$ if 
    \begin{itemize}
        \item[1.] $f$ is a self-concordant function on $\interior(\X)$;
        \item[2.] for all $x\in \interior(\X)$ and $u \in \R^d$,
        \[
        |Df(x)[u]| \le \theta^\frac{1}{2}(D^2f(x)[u,u])^{1/2}.
        \]
    \end{itemize}
\end{definition}
As an example, the log barrier for the non-negative ray, i.e., $f : (0,\infty) \owns x \rightarrow -\log x$, is a $1$-self-concordant barrier. 

Next, we introduce the log barrier regularizer on the simplex. To address the issue that the simplex $\Delta^d$ has empty interior, we will restrict the problem to the domain $\Delta^\circ:= \{x \in \R^{d-1}_{\ge 0}: \sum_{r=1}^{d-1} x[r] \le 1\}$. For notational convenience, we also denote $x[d] = 1 - \sum_{r=1}^{d-1} x[r]$.  The log barrier regularizer for $\Delta^\circ$ is defined as follows.
\begin{definition}[Log Barrier Regularizer for the Simplex]
    For $x \in \Delta^\circ$, the log barrier regularizer is 
    \begin{equation}\label{eq:log barrier}
        \+R(x):= -\sum_{r=1}^{d-1} \log(x[r]) - \log(1 - \sum_{r=1}^{d-1}x[r]).
    \end{equation}
\end{definition}
It can be shown that $\+R$ is a $d$-self-concordant barrier. Since the regualarizer $\+R$ takes a $(d-1)$-dimensional vector as input while the regret minimizer receives a $d$-dimensional utility vector $u \in \R^d$, we first explain how the regret minimizer operates on $\Delta^\circ$. Upon receiving a utility vector $u \in \R^d$, the algorithm first constructs $\Tilde{u} \in \R^{d-1}$ so that $\Tilde{u}[r] = u[r] - u[d]$, for all $r \in [d-1]$. It is clear that the regret incurred is preserved after the transformation. For the purpose of analysis, we also introduce an auxiliary regularizer $\Tilde{\+R}$:
\begin{equation}
    \Tilde{\+R}(x) := -\sum_{r=1}^d \log x[r].
\end{equation}
The following claim characterizes and relates the local norm induced by $\+R$ and $\Tilde{\+R}$
\begin{claim}[\citep{anagnostides2022uncoupled}]
\label{claim:logbar local norm}
    For any $x, x' \in \interior(\Delta^\circ)$.
    \[
    \InNorms{x - x'}_{\+R,x}^2 = \sum_{r=1}^d\InParentheses{\frac{x[r]  - x'[r]}{x[r]}}^2.
    \]
    For any $\Tilde{u} \in \R^{d-1}$ and $x \in \interior(\Delta^\circ)$, 
    \[
    \InNorms{\Tilde{u}}_{*,\+R,x} \le \InNorms{u}_{*,\Tilde{\+R},x}.
    \]
\end{claim}

\section{Missing Proofs in \Cref{sec:bounding regret}}
We recall the update rule of Optimistic Follow the Regularized Leader \eqref{OFTRL} algorithm. In this section, we focus \eqref{OFTRL} with decreasing step size $\eta_t = \frac{\eta}{w_t}$ where $w_t = \frac{\alpha^t_t}{\alpha^1_t}$ for  all $t\ge 1$ (see \Cref{eq:alpha_t},  \eqref{eq:alpha_t^i} and \eqref{eq:w_t} for definitions). Moreover, we remark that we usually write the utility and prediction vectors in the form of $u^t = w_t \hu^t$, and $m^t = w_t \hat{m}^t $ such that $\InNorms{\hu^t}_{\infty}, \InNorms{\hat{m}^t}_\infty \le 1$. We define $\eta_0 = \eta_1$, $w_0 = w_1$, and $x^0 := \argmin_{x \in \X} \+R(x)$.
\begin{equation}
\tag{OFTRL}
    \begin{aligned}
        x^{t} = \argmax_{x \in \X} \left\{ \Phi^t(x):= \eta_t \InAngles{x, m^t + \sum_{\tau=1}^{t-1} u^\tau} - \+R(x) \right\}
    \end{aligned}
\end{equation}

We also define $\{g^t\}$ as the sequence produced by the conceptual algorithm Be-the-Leader \eqref{BTL}, which updates $g^t$ with the information of utility vector $u^t$. 
\begin{equation}
\label{BTL}
\tag{BTL}
    \begin{aligned}
        g^{t} = \argmax_{g \in \X} \left\{ \Psi^t(g):= \eta_t \InAngles{g, \sum_{\tau=1}^{t} u^\tau} - \+R(g) \right\}
    \end{aligned}
\end{equation}
We remark that both \eqref{OFTRL} and \eqref{BTL} are well-defined if $\eta_t \InNorms{u^t - m^t}_{*, x^t} \le \frac{1}{2}$ and $\InNorms{\eta_t m^t + (\eta_t - \eta_{t-1}) \sum_{\tau=1}^{t-1} u^\tau }_{*, g^{t-1}} \le \frac{1}{2}$ for all $t \in [T]$, which can be verified using \Cref{lemma:unique solution self-concordant} and \Cref{lemma: stability}.

\subsection{RVU for \eqref{OFTRL} with Decreasing Step Sizes}

\begin{theorem}[Adapted from Theorem B.1 in \citep{anagnostides2022uncoupled}]
\label{thm:OFTRL self-concordant}
    Suppose that $\+R$ is a non-degenerate self-concordant barrier function for $\interior(\+X)$ and let $\eta > 0$. Then,the regret of \ref{OFTRL} with respect to any $x^* \in \interior(\X)$ and under any sequence of utilities $u^1, \ldots, u^T$ can be bounded as
    \begin{align*}
        \frac{R(x^*)}{\eta_T} + \sum_{t=1}^T \InNorms{u^t - m^t}_{*, x^t} \InNorms{x^t - g^t}_{x^t} - \sum_{t=1}^T \frac{1}{\eta_t} \omega\InParentheses{\InNorms{x^t - g^t}_{x^t}} - \frac{1}{\eta_{t-1}} \omega\InParentheses{\InNorms{x^t - g^{t-1}}_{g^{t-1}}}
    \end{align*}
    where the function $\omega(\cdot)$ is defined in \Cref{fact:omega(s)}.
\end{theorem}
\begin{lemma}[Stability]
\label{lemma: stability}
    Let $\eta_t > 0$ be such that $\eta_t \InNorms{u^t - m^t}_{*, x^t} \le \frac{1}{2}$ and $\InNorms{\eta_t m^t + (\eta_t - \eta_{t-1}) \sum_{\tau=1}^{t-1} u^\tau }_{*, g^{t-1}} \le \frac{1}{2}$. Then we have
    \begin{align*}
        \InNorms{x^t - g^t}_{x^t} &\le 2\eta_t \InNorms{u^t - m^t}_{*, x^t}, \\
        \InNorms{x^t - g^{t-1}}_{g^{t-1}} &\le 2\InNorms{\eta_t m^t + (\eta_t - \eta_{t-1}) \sum_{\tau=1}^{t-1} u^\tau }_{*, g^{t-1}}.
    \end{align*}
\end{lemma} 
\begin{proof}
    Fix any $t\in [T]$. We first note that $x^t - g^t = x^t - \argmin\{-\Psi^t\}$ by the definition of $\Psi^t$ in \eqref{BTL}.  We also have $\Psi^t(x) = \Phi^t(x) + \eta_t\InAngles{x, u^t - m^t}$. Thus we have
    \[
    \nabla \Psi^t(x^t) =  \nabla \Phi^t(x^t) + \eta_t (u^t - m^t) = \eta_t (u^t - m^t),
    \]
    where the second inequality holds since $\nabla \Phi^t(x^t) = 0$ by the first order optimality condition of the optimization problem associated with \eqref{OFTRL}. By assumption, it further implies that $\InNorms{\nabla \Psi^t(x^t)}_{*,x^t} = \eta_t\InNorms{u^t - m^t}_{*,x^t} \le \frac{1}{2}$. Now we can use \Cref{lemma:local norm} to get 
    \[
    \InNorms{x^t -g^t}_{x^t} = \InNorms{x^t - \argmin\{-\Psi^t\} }_{x^t} \le 2 \InNorms{\nabla\Psi^t(x^t)}_{*,x^t} = 2\eta_t \InNorms{u^t - m^t}_{*, x^t}.
    \]
    This finishes the proof of the first inequality. 
    
    The proof of the second inequality follows the same idea. We first note that $x^t - g^{t-1} = \argmin\{-\Phi^t(x)\} - g^{t-1}$ by the definition of $\Phi^t$ in \eqref{OFTRL}.  We also have $\Phi^t(x) = \Psi^{t-1}(x) + \InAngles{x, \eta_t m^t + (\eta_t - \eta_{t-1})\sum_{\tau=1}^{t-1}  u^\tau}$. Thus we have
    \[
    \nabla \Phi^t(g^{t-1}) =  \nabla \Psi^{t-1}(g^{t-1}) + \eta_t m^t + (\eta_t - \eta_{t-1})\sum_{\tau=1}^{t-1}  u^\tau = \eta_t m^t + (\eta_t - \eta_{t-1})\sum_{\tau=1}^{t-1}  u^\tau,
    \]
    where the second inequality holds since $\nabla \Psi^{t-1}(g^{t-1}) = 0$ by the first order optimality condition of the optimization problem associated with \eqref{BTL}. By assumption, it further implies that $\InNorms{\nabla \Phi^t(g^{t-1})}_{*,g^{t-1}} = \InNorms{\eta_t m^t + (\eta_t - \eta_{t-1}) \sum_{\tau=1}^{t-1} u^\tau }_{*, g^{t-1}} \le \frac{1}{2}$. Now we can use \Cref{lemma:local norm} to get 
    \[
    \InNorms{x^t -g^{t-1}}_{g^{t-1}} = \InNorms{g^{t-1} - \argmin\{-\Phi^t\} }_{g^{t-1}} \le 2 \InNorms{\nabla\Phi^t(g^{t-1})}_{*,g^{t-1}} = 2\InNorms{\eta_t m^t + (\eta_t - \eta_{t-1}) \sum_{\tau=1}^{t-1} u^\tau }_{*, g^{t-1}}.
    \]
    This finishes the proof of the second inequality. 
\end{proof}

\subsection{Proof of \Cref{thm: RVU-stable}}
\begin{proof}
    Combining \Cref{thm:OFTRL self-concordant} with the fact that $\InNorms{x^t - g^t}_{x^t} \le 2\eta_t \InNorms{u^t - m^t}_{*,x^t}$ (by \Cref{lemma: stability}) gives 
    \[
        \Reg^T(x^*) \le \frac{R(x^*)}{\eta_T} + \sum_{t=1}^T2\eta_t \InNorms{u^t - m^t}_{*, x^t}^2  - \sum_{t=1}^T \frac{1}{\eta_t} \omega\InParentheses{\InNorms{x^t - g^t}_{x^t}} - \frac{1}{\eta_{t-1}} \omega\InParentheses{\InNorms{x^t - g^{t-1}}_{g^{t-1}}}.
    \]
    Moreover, by \Cref{lemma: stability}, we know $\InNorms{x^t - g^t}_{x^t} \le 1$ and $\InNorms{x^t - g^{t-1}}_{g^{t-1}} \le 1$. Then by \Cref{fact:omega(s)}, we get 
    \[
        \Reg^T(x^*) \le \frac{R(x^*)}{\eta_T} + 2\sum_{t=1}^T\eta_t \InNorms{u^t - m^t}_{*, x^t}^2 - \sum_{t=1}^T \InParentheses{\frac{1}{4\eta_t} \InNorms{x^t - g^t}_{x^t}^2 + \frac{1}{4\eta_{t-1}} \InNorms{x^t - g^{t-1}}_{g^{t-1}}^2}.
    \]
\end{proof}

The following corollary is useful to apply the RVU property to \eqref{OFTRL} with the log barrier regularization.
\begin{corollary}
\label{corollary: RVU-stable}
    Suppose that $\+R$ is a non-degenerate self-concordant barrier function for $\interior(\+X)$ such that $\nabla^2 \+R(\Tilde{x}) \preceq 2\nabla^2 \+R(x)$ for any $x, \Tilde{x} \in \interior(\+X)$ with $\InNorms{x -\Tilde{x}}_{\Tilde{x}}\le \frac{1}{4}$. Moreover, let $\eta_t > 0$ be such that $\eta_t \InNorms{u^t - m^t}_{*, x^t} \le \frac{1}{8}$ and $\InNorms{\eta_t m^t + (\eta_t - \eta_{t-1}) \sum_{\tau=1}^{t-1} u^\tau }_{*, g^{t-1}} \le \frac{1}{2}$ for all $t \in [T]$. Then, the regret of \ref{OFTRL} with respect to any $x^* \in \interior(\X)$ and under any sequence of utilities $u^1, \ldots, u^T$ can be bounded as
    \begin{align*}
        \Reg^T(x^*) \le \frac{R(x^*)}{\eta_T} + 2\sum_{t=1}^T\eta_t \InNorms{u^t - m^t}_{*, x^t}^2 - \sum_{t=1}^T \frac{1}{16\eta_{t-1}} \InNorms{x^t - x^{t-1}}_{x^{t-1}}^2. 
    \end{align*}
\end{corollary}
\begin{proof}
    By \Cref{lemma: stability}, we have $\InNorms{x^{t-1}- g^{t-1}}_{x^{t-1}}\le 2\eta_{t-1} \InNorms{u^{t-1} - m^{t-1}}_{*, x^{t-1}} \le \frac{1}{4}$ for all $t$. Thus by assumption, we have$\nabla^2 \+R(\Tilde{x^{t-1}}) \preceq 2\nabla^2 \+R(g^{t-1})$. It further implies $\InNorms{x^t - g^{t-1}}_{x^{t-1}} \le 2 \InNorms{x^t - g^{t-1}}_{g^{t-1}}$. Thus we have
    \begin{align*}
        \InNorms{x^t - x^{t-1}}^2_{x^{t-1}} &\le 2\InNorms{x^t - g^{t-1}}^2_{x^{t-1}} + 2\InNorms{x^{t-1} - g^{t-1}}^2_{x^{t-1}}\\
        &\le 4\InNorms{x^t - g^{t-1}}^2_{g^{t-1}} + 4\InNorms{x^{t-1} - g^{t-1}}^2_{x^{t-1}}.
    \end{align*}
    Since the step size $\{\eta_t\}$ is non-increasing, we have
    \begin{align*}
         \sum_{t=1}^T \InParentheses{\frac{1}{4\eta_t} \InNorms{x^t - g^t}_{x^t}^2 + \frac{1}{4\eta_{t-1}} \InNorms{x^t - g^{t-1}}_{g^{t-1}}^2} &\ge \sum_{t=1}^T \InParentheses{\frac{1}{4\eta_{t-1}} \InNorms{x^t - g^t}_{x^t}^2 + \frac{1}{4\eta_{t-1}} \InNorms{x^t - g^{t-1}}_{g^{t-1}}^2} \\
         &\ge \sum_{t=1}^T \frac{1}{16\eta_{t-1}} \InNorms{x^t - x^{t-1}}^2_{x^{t-1}}.
    \end{align*}
    Combining the above with \Cref{thm: RVU-stable} finishes the proof.
\end{proof}

\begin{corollary}
\label{stablity of iterates}
    Suppose that $\+R$ is a non-degenerate self-concordant barrier function for $\interior(\+X)$ such that $\nabla^2 \+R(\Tilde{x}) \le 2\nabla^2 \+R(x)$ for any $x, \Tilde{x} \in \interior(\+X)$ with $\InNorms{x -\Tilde{x}}_{\Tilde{x}} \le \frac{1}{4}$. Moreover, let $\eta_t > 0$ be such that $\eta_t \InNorms{u^t - m^t}_{*, x^t} \le \frac{1}{8}$ and $\InNorms{\eta_t m^t + (\eta_t - \eta_{t-1}) \sum_{\tau=1}^{t-1} u^\tau }_{*, g^{t-1}} \le \frac{1}{2}$ for all $t \in [T]$. Then
    \begin{align*}
        \InNorms{x^t - x^{t-1}}_{x^{t-1}} \le 14 \eta.
    \end{align*}
\end{corollary}
\begin{proof}
    Similar to the proof of \Cref{corollary: RVU-stable}, we have
    \begin{align*}
        \InNorms{x^t - x^{t-1}}_{x^{t-1}} &\le  \InNorms{x^t - g^{t-1}}_{x^{t-1}} +  \InNorms{g^{t-1} - x^{t-1}}_{x^{t-1}}\\
        &\le 2\InNorms{x^t - g^{t-1}}_{g^{t-1}} +  \InNorms{g^{t-1} - x^{t-1}}_{x^{t-1}}\\
        &\le 4 \InNorms{\eta_t m^t + (\eta_t - \eta_{t-1}) \sum_{\tau=1}^{t-1} u^\tau }_{*, g^{t-1}} + 2 \eta_{t-1} \InNorms{u^{t-1} - m^{t-1}}_{*, x^{t-1}} \\
        &\le 14 \eta \tag{by \Cref{lemma:stability of eta}}.
    \end{align*}
\end{proof}

\subsection{Proof of \Cref{corollary: RVU-logbar}: RVU for \eqref{OFTRL-LogBar}}
Combining \Cref{corollary: RVU-stable} and \Cref{claim:logbar local norm}  directly leads to the RVU bound in \Cref{corollary: RVU-logbar}.

\subsection{Proof of \Cref{lemma: l1norm}}
\begin{proof}
    The proof follows from the proof of Lemma 4.2 in \cite{anagnostides2022uncoupled}, which shows that it suffices to prove $\sum_{a \in \+A} \mu^t_a \le \frac{1}{2}$ where $\mu^t_a$ is defined as 
    \begin{align*}
        \mu^t_a := \max_{a' \in \+A} \left | 1 - \frac{x^t_a[a']}{x^{t-1}_a[a']} \right|. 
    \end{align*}
    Recall the local norm induced by the log barrier regularization (\Cref{claim:logbar local norm}):  $\InNorms{x - x'}_{x}^2 = \sum_{r=1}^d \InParentheses{\frac{x[r] - x'[r]}{x[r]}}^2$. Thus
    \begin{align*}
        \mu^t_a = \max_{a' \in \+A} \left | 1 - \frac{x^t_a[a']}{x^{t-1}_a[a']} \right| \le \sqrt{\sum_{a' \in\+A} \InParentheses{  \frac{x^{t-1}_a[a'] - x^t_a[a']}{x^{t-1}_a[a']} }^2} = \InNorms{x^t_a - x^{t-1}_a}_{x^{t-1}_a}.
    \end{align*}
    Now combining \Cref{stablity of iterates} and $\eta < \frac{1}{28 |\+A|}$, we get
    \begin{align*}
        \sum_{a \in \+A} \mu^t_a \le |\+A| \max_{a \in \+A} \InNorms{x^t_a - x^{t-1}_a}_{x^{t-1}_a} \le 14 |\+A| \eta \le \frac{1}{2}.
    \end{align*}
\end{proof}

\subsection{Proof of \Cref{thm:RVU for swap regret}: RVU for Swap Regret}
\begin{proof}
    Applying \Cref{corollary: RVU-stable} and  to each regret minimizer $\Re_a$ gives the following regret guarantee for all $\eta \le \frac{1}{28 m}$:
    \begin{align*}
        \Reg^T_a(x^*_a) &\le \frac{R(x^*_a)}{\eta_T} + 2\sum_{t=1}^T\eta_t \InNorms{u^t x^t[a] - m^tx^t[a]}_{*, x^t_a}^2 - \sum_{t=1}^T \frac{1}{16\eta_{t-1}} \InNorms{x^t_a - x^{t-1}_a }_{x^{t-1}_a}^2 \\
        &\le \frac{R(x^*_a)}{\eta_T} + 2\sum_{t=1}^T\eta_t (x^t[a])^2 \InNorms{u^t - m^t}_{\infty}^2 - \sum_{t=1}^T \frac{1}{16\eta_{t-1}} \InNorms{x^t_a - x_a^{t-1} }_{x^{t-1}_a}^2
    \end{align*}
    for any $x^*_a \in \relint(\Delta_{\+A})$. Following the same argument in \cite[Lemma 4.2]{anagnostides2022uncoupled}, we can also bound the diameter term $R(x^*_a)$ and get
    \begin{align*}
        \Reg^T_a &\le \frac{|\+A|\log T}{\eta_T} + \frac{2}{T} \sum_{t=1}^T x^t[a] \InNorms{u^t}_\infty + 2\sum_{t=1}^T\eta_t (x^t[a])^2 \InNorms{u^t - m^t}_{\infty}^2 - \sum_{t=1}^T \frac{1}{16\eta_{t-1}} \InNorms{x^t_a - x^{t-1}_a }_{x^{t-1}_a}^2 \\
        &\le \frac{|\+A|\log T}{\eta_T} + 2w_t + 2\sum_{t=1}^T\eta_t (x^t[a])^2 \InNorms{u^t - m^t}_{\infty}^2 - \sum_{t=1}^T \frac{1}{16\eta_{t-1}} \InNorms{x^t_a - x^{t-1}_a }_{x^{t-1}_a}^2 \\
        &\le \frac{2|\+A|\log T}{\eta_T} + 2\sum_{t=1}^T\eta_t (x^t[a])^2 \InNorms{u^t - m^t}_{\infty}^2 - \sum_{t=1}^T \frac{1}{16\eta_{t-1}} \InNorms{x^t_a - x^{t-1}_a }_{x^{t-1}_a}^2,
    \end{align*}
    where we use $w_T = \frac{\eta}{\eta_T} \le \frac{1}{\eta_T}$ in the last inequality. 
    Summing the above inequality for all $a \in \+A$ and applying \Cref{lemma: l1norm} gives
    \begin{align*}
        \sreg^T \le \sum_{a \in \+A} \reg^T_a \le \frac{2|\+A|^2\log T}{\eta_T} + 2 \sum_{t=1}^T \eta_t \InNorms{u^t - m^t}_\infty^2 - \sum_{t=1}^T \frac{1}{1024|\+A| \eta_{t-1}}\InNorms{x^t- x^{t-1}}^2_1.
    \end{align*}
\end{proof}

\end{document}

%% file: macro.tex
\newtheorem{definition}{Definition}
\newtheorem{theorem}{Theorem}
\newtheorem*{theorem*}{Theorem}
\newtheorem{lemma}{Lemma}
\newtheorem{fact}{Fact}
\newtheorem{corollary}{Corollary}

\newtheorem{claim}{Claim}
\newtheorem{proposition}{Proposition}
\newtheorem{remark}{Remark}

\newcommand{\reg}{\mathrm{reg}}
\newcommand{\regbar}{\overline{\reg}}
\newcommand{\Reg}{\mathrm{Reg}}
\newcommand{\sreg}{\mathrm{SwapReg}}
\newcommand{\cegap}{\mathrm{CEGap}}
\usepackage{pifont}

\DeclareMathOperator*{\argmin}{argmin}
\DeclareMathOperator*{\argmax}{argmax}
\DeclareMathOperator{\poly}{poly}

\DeclareMathOperator{\interior}{int}
\DeclareMathOperator{\relint}{relint}

\newcommand{\notshow}[1]{{}}
\newcommand{\AutoAdjust}[3]{{ \mathchoice{ \left #1 #2  \right #3}{#1 #2 #3}{#1 #2 #3}{#1 #2 #3} }}
\newcommand{\Xcomment}[1]{{}}

\newcommand{\InParentheses}[1]{\AutoAdjust{(}{#1}{)}}
\newcommand{\InBrackets}[1]{\AutoAdjust{[}{#1}{]}}
\newcommand{\InAngles}[1]{\AutoAdjust{\langle}{#1}{\rangle}}
\newcommand{\InNorms}[1]{\AutoAdjust{\|}{#1}{\|}}
\renewcommand{\part}[2]{\frac{\partial #1}{\partial #2}}

\newcommand{\R}{\mathds{R}}

\newcommand{\X}{\mathcal{X}}

\newcommand{\hu}{\hat{u}}

\newcommand{\hpi}{\widehat{\pi}}

\newcommand{\bolda}{\boldsymbol{a}}

\def\+#1{\mathcal{#1}}
\def\-#1{\mathbb{#1}}